\documentclass[11pt]{article}
\usepackage[left=1in, right=1in, top=1in, bottom=1in]{geometry}   
\usepackage{fancyhdr}   
\usepackage[utf8]{inputenc} 
\usepackage[T1]{fontenc}    
\usepackage{hyperref}       
\usepackage{url}            
\usepackage{booktabs}       
\usepackage{amsfonts}       
\usepackage{nicefrac}       
\usepackage{microtype}      
\usepackage{xcolor}         
\usepackage{siunitx}
\usepackage{makecell}   

\newcommand{\meanstd}[2]{%
\makecell[c]{\num{#1}\\[-2pt]\scriptsize$\pm$\,\num{#2}}%
}
\newcommand{\runtime}[1]{%
  \num[round-precision=2]{#1}}

\usepackage{amsmath}
\usepackage{algorithm}
\usepackage{algpseudocode}
\usepackage{amsthm}
\usepackage{amssymb}
\usepackage{graphicx}
\usepackage{verbatim} 
\usepackage{authblk}

\newtheorem{theorem}{Theorem}
\newtheorem{lemma}{Lemma}
\newtheorem{definition}{Definition}

\newtheorem{corollary}{Corollary}

\newtheorem{remark}{Remark}

\newcommand\fro[1]{\| #1 \|_{\rm{F}}}
\newcommand\op[1]{\| #1 \|}
\newcommand\nuc[1]{\| #1 \|_{*}}
\newcommand\lzero[1]{\| #1 \|_{0}}
\newcommand\lone[1]{\| #1 \|_{\ell_1}}

\newcommand\linf[1]{\| #1 \|_{\ell_\infty}}
\newcommand\psitwo[1]{\| #1 \|_{\psi_2}}
\newcommand{\mat}[1]{\begin{bmatrix}#1 \\ \end{bmatrix}}
\newcommand{\inp}[2]{\langle #1,#2\rangle}

\def\calA{{\mathcal A}}
\def\calB{{\mathcal B}}

\def\calD{{\mathcal D}}
\def\calE{{\mathcal E}}
\def\calF{{\mathcal F}}

\def\calP{{\mathcal P}}

\def\calS{{\mathcal S}}
\def\calT{{\mathcal T}}

\def\calY{{\mathcal Y}}

\def\bcalT{{\boldsymbol{\mathcal T}}}

\def\bcalY{{\boldsymbol{\mathcal Y}}}
\def\bcalZ{{\boldsymbol{\mathcal Z}}}

\def\CC{{\mathbb C}}

\def\EE{{\mathbb E}}

\def\PP{{\mathbb P}}

\def\RR{{\mathbb R}}

\def\TT{{\mathbb T}}

\def\e{{\boldsymbol e}}
\def\f{{\boldsymbol f}}

\def\r{{\boldsymbol r}}

\def\x{{\boldsymbol x}}
\def\y{{\boldsymbol y}}

\def\A{{\boldsymbol A}}
\def\B{{\boldsymbol B}}

\def\F{{\boldsymbol F}}
\def\G{{\boldsymbol G}}

\def\M{{\boldsymbol M}}

\def\S{{\boldsymbol S}}
\def\T{{\boldsymbol T}}
\def\U{{\boldsymbol U}}
\def\V{{\boldsymbol V}}
\def\W{{\boldsymbol W}}
\def\X{{\boldsymbol X}}
\def\Y{{\boldsymbol Y}}

\def\rank{\textsf{rank}}
\def\diag{\textsf{diag}}
\def\supp{\textsf{supp}}

\def\sign{\textsf{sign}}
\def\conj{\textsf{conj}}
\def\dof{\textsf{dof}}
\def\re{\textsf{Re}}
\def\im{\textsf{Im}}

\def\bSigma{{\boldsymbol \Sigma}}
\def\name{\textsf{FLoST}}

\title{Fourier Low-rank and Sparse Tensor for Efficient Tensor Completion}

\author{
	Jingyang Li, Jiuqian Shang and Yang Chen \\
	University of Michigan, Ann Arbor 
}
\date{}

\begin{document}

\maketitle

\begin{abstract}
Tensor completion is crucial in many scientific domains with missing data problems. Traditional low-rank tensor models, including CP, Tucker, and Tensor-Train, exploit low-dimensional structures to recover missing data. However, these methods often treat all tensor modes symmetrically, failing to capture the unique spatiotemporal patterns inherent in scientific data, where the temporal component exhibits both low-frequency stability and high-frequency variations. To address this, we propose a novel model, \underline{F}ourier \underline{Lo}w-rank and \underline{S}parse \underline{T}ensor (\name), which decomposes the tensor along the temporal dimension using a Fourier transform. This approach captures low-frequency components with low-rank matrices and high-frequency fluctuations with sparsity, resulting in a hybrid structure that efficiently models both smooth and localized variations. Compared to the well-known tubal-rank model, which assumes low-rankness across all frequency components, \name~ requires significantly fewer parameters, making it computationally more efficient, particularly when the time dimension is large. Through theoretical analysis and empirical experiments, we demonstrate that \name~ outperforms existing tensor completion models in terms of both accuracy and computational efficiency, offering a more interpretable solution for spatiotemporal data reconstruction.
\end{abstract}

\section{Introduction}
\label{sec:Introduction}
Tensor completion plays a critical role in many scientific and engineering domains, such as visual data in-painting \cite{liu2012tensor}, medical imaging \cite{wang2020learning},  to name a few. In the natural sciences, one prominent example is the recovery of total electron content (TEC) maps \cite{rideout2006automated}, which represent ionospheric electron distributions across space and time. Due to limitations in sensor coverage and data transmission, only a subset of the full TEC tensor is observed at any given time. Accurate completion of these tensors is essential for downstream tasks such as the Positioning, Navigation and Timing (PNT) services.

To tackle the problem of tensor completion, a common strategy is to exploit low-rank structures that capture the intrinsic correlations within the tensor. Several low-rank tensor models have been proposed in the literature, including CP rank \cite{jain2014provable,cai2022nonconvex}, Tucker rank \cite{xia2019polynomial,xia2021statistically,li2023online}, Tensor-Train (TT) rank \cite{cai2022provable,cai2022tensor}, and tubal-rank \cite{zhang2016exact,liu2019low,jiang2019robust,song2023riemannian}. 
By leveraging various low-dimensional structures, these models enable accurate recovery of the underlying tensor from partial observations.

Selecting an appropriate low-rank structure remains a central challenge in tensor completion. Existing formats such as CP, Tucker, and TT typically treat all tensor modes symmetrically, without distinguishing between spatial and temporal dimensions. In our motivating TEC completion task, each tube (i.e., lateral slice along the time dimension) represents the temporal evolution of measurements at a fixed spatial location. In most regions, ionospheric activity is relatively stable and varies slowly, corresponding to stationary low-frequency patterns. In contrast, certain regions affected by strong solar or geomagnetic disturbances exhibit rapid and abrupt fluctuations. While the well-known tubal-rank structure does incorporate temporal information by operating in the Fourier domain, it still assumes that every frequency slice is low-rank. This assumption introduces unnecessary redundancy, especially for TEC data, where only the low-frequency components are structured and the high-frequency components are sparse and localized.

\begin{figure}[t]
    \centering
     \includegraphics[width=0.8\linewidth]{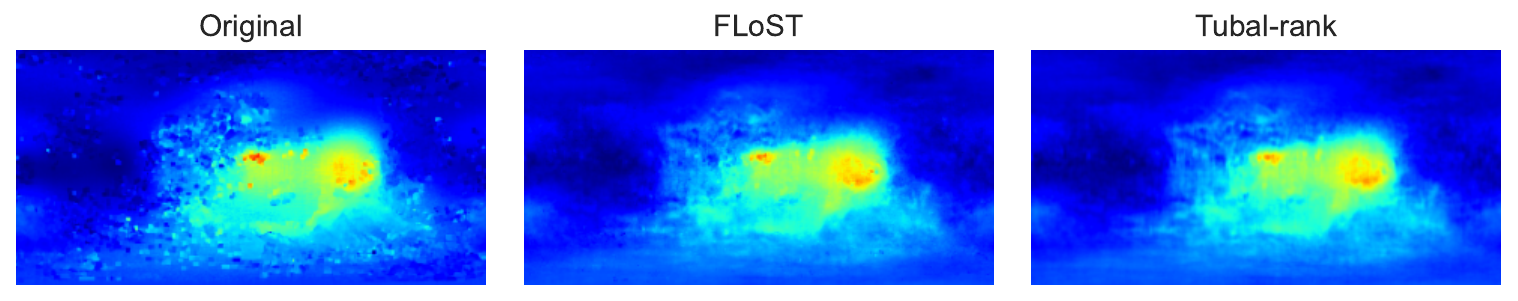}
    \caption{Comparison of truncation results using the proposed \name~model and the tubal-rank model on TEC data.}
    \label{fig:truncation}
\end{figure}

To address this limitation, we propose a novel model: \underline{F}ourier \underline{Lo}w-rank and \underline{S}parse \underline{T}ensor (\name). The key idea is to decompose the tensor along the temporal mode via a Fourier transform, yielding frequency-domain slices that are naturally ordered by temporal frequency. The initial slices capture smooth, slowly varying components, while the later ones encode rapid, high-frequency fluctuations. In \name, we model the low-frequency slices using low-rank matrices to efficiently represent stationary regions, and represent the high-frequency slices—corresponding to infrequent but sharp variations—using sparse matrices. This hybrid structure captures both the global smoothness of the TEC field and the localized temporal anomalies in a flexible and compact way. Compared to the tubal-rank model, which assumes low-rankness across all frequency components, \name~ requires significantly fewer parameters, especially when the temporal dimension is large. 
For example, a typical TEC dataset collected over one month forms a tensor of size $181 \times 361 \times 8640$. In such cases, the tubal-rank model, which assumes low-rankness across all frequency slices, introduces substantial redundancy. In contrast, \name~requires significantly fewer parameters by focusing modeling capacity on the informative low-frequency components, resulting in improved efficiency and better alignment with the physical characteristics of the data.
A visual comparison of truncation results shows that \name~ effectively captures the essential spatiotemporal structures with far fewer parameters(see Figure~\ref{fig:truncation}).
Both theoretical analysis and empirical results demonstrate that \name~ outperforms existing tensor formats in terms of accuracy and efficiency.

\section{Methodology}
In this section, we introduce the proposed \name~ structure, which models low-frequency components as low-rank and high-frequency components as sparse. Based on this structure, we further develop an efficient estimator for recovering the underlying tensor from partial and noisy observations.
\subsection{Notations}
For a general third order tensor $\calT\in\RR^{M\times N \times T}$, we denote $\bcalT\in\RR^{M\times N\times T}$ as a result of fast Fourier transform of $\calT$ along the third dimension.
The multi-linear product between $\calT$ and matrix $M\in\RR^{L\times T}$, is defined by $[\calT\times_3 M]_{ijk} = \sum_{l=1}^T[\calT]_{ijl}M_{kl}$. Using the notation of multi-linear product, we can write $\bcalT = \calT\times_3\F$, where $\F\in\RR^{T\times T}$ is the discrete Fourier transform matrix, i.e., $[\F]_{ij} = \gamma^{(i-1)(j-1)}/\sqrt{T}$, with $\gamma = e^{-2\pi i/T}$.  We use $\f_l$ to denote the $l$-th column of $\F$, that is $\F = \mat{\f_1 &\cdots&\f_T}$. 
In the following, we would use bold-face letters to denote matrix/tensor in Fourier domain. 
For $\x\in\CC^T$, we define its $\ell_1$ norm as $\|\x\|_{\ell_1} = \sum_{i=1}^T\big(|\re(x_i)| + |\im(x_i)|\big)$, and $\|\x\|_{\ell_2} = \big(\sum_{i=1}^T|x_i|^2 \big)^{1/2}$, $\|\x\|_{\ell_{\infty}} = \max_{i=1}^T|x_i|$.
We also denote $\lzero{\calT}$ to be the number of non-zero entries of $\calT$. 
For a complex matrix $\M\in\CC^{M\times N}$, the spectral norm of $\M$ is denoted as $\op{\M} = \max_i\sigma_i(\M)$, where $\sigma_i(\M)$ are the singular values of $\M$, and $\nuc{\M} = \sum_i\sigma_i(\M)$, 
we also use $\conj(\M)$ to denote the conjugate of $\M$. 
We use $\fro{\cdots}$ to denote the Frobenius norm of matrices/tensors.
For a real number $x$, we use $\lceil x\rceil$ to denote the smallest integer that is greater than or equal to $x$. We denote $C, C_1, C_2, c, c_1, c_2, \cdots$ some absolute constants whose actual values might vary at different appearances. We also use $a\vee b$ to represent $\max\{a,b\}$ and $a\wedge b$ to represent $\min\{a,b\}$. 

\subsection{Fourier Low-rank and Sparse Tensor}
Among the various low-rank tensor models, the tubal-rank framework has gained significant attention due to its ability to exploit temporal structures by transforming a third-order tensor along its temporal mode via the discrete Fourier transform.
However, despite its advantage in leveraging temporal information, the tubal-rank model assumes all frequency components are low-rank, leading to unnecessary redundancy in real data, where typically only low-frequency components are structured and high-frequency components are sparse. This motivates us to propose a hybrid low-rank and sparse structure for more efficient frequency-domain modeling (see Figure \ref{fig:flost}). Since $\calT\in\RR^{M\times N\times T}$ is real-valued, its Fourier transform satisfies the conjugate symmetry property: $$\calT\times_3 \f_l^\top = \conj(\calT\times_3 \f_{T+2-l}^\top)$$ for $l=\lceil\frac{T+1}{2}\rceil+1,\cdots, T$. As a result, it suffices to define only the first $\lceil \tfrac{T+1}{2} \rceil$ frequency slices.
\begin{definition}
	We say a tensor $\calT\in\RR^{M\times N\times T}$ is $(\r,K,s)$-\name~ for $\r = (r_1,\cdots,r_K)$, if 
	\begin{align*}
		\rank(\calT\times_3 \f_l^\top)\leq r_l, ~l = 1,\cdots, K,\quad \lzero{\calT\times_3 \F_1^\top}\leq s,
	\end{align*}
	where $\F_1 = [\f_{K+1}, \cdots, \f_{\lceil\frac{T+1}{2}\rceil}]$ and $K\leq \lceil\frac{T+1}{2}\rceil$.
\end{definition}
\begin{figure}[H]
    \centering
    \includegraphics[width=0.75\linewidth]{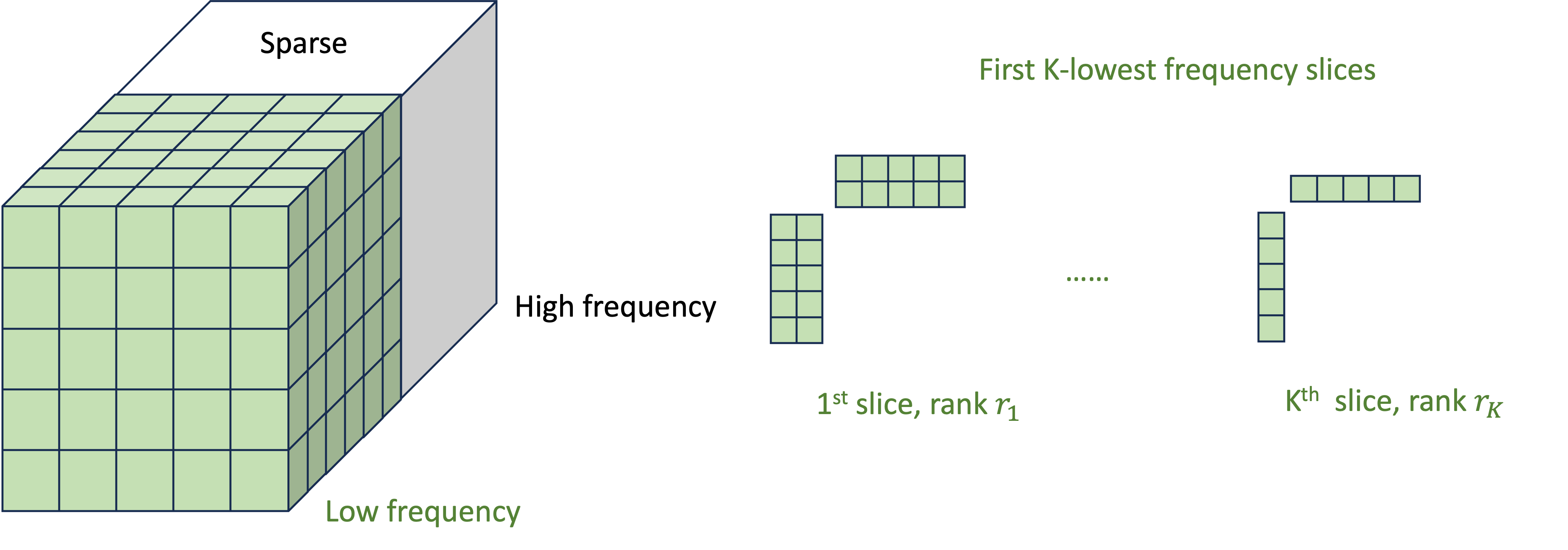}
    \caption{Structure of \name~ in frequency domain. }
    \label{fig:flost}
\end{figure}
Here, $\calT \times_3 \f_l^\top$ extracts the $l$-th frontal slice of the tensor in the frequency domain, and $\calT\times_3 \F_1^\top$ corresponds to the collection of high-frequency slices. 
\begin{remark}[Connections with Multi-rank and tubal-rank]
The proposed \name~ model generalizes several existing tensor structures. Specifically, when $K = \lceil\frac{T+1}{2}\rceil$, it reduces to the multi-rank model if the frequency slices are allowed to have different ranks, and the tubal-rank model if all frequency slices share the same rank \cite{zhang2016exact}.
\end{remark}

\subsection{\name~ for Tensor Completion} \label{sec:loss}
The goal of tensor completion is to recover an underlying tensor by only observing a small subset of its entries. 
Let $\calT_0\in\RR^{M\times N \times T}$ be a tensor that is $(\r,K,s)$-\name. 
We can obtain a collection of partially observed and contaminated entries where the sampling indicator $\omega_{ijt}=1$:
\begin{align*}
	\calY_{ijt} = [\calT_0]_{ijt}+\calE_{ijt}.
\end{align*}
Here $\calE$ represents the random noise independent of $\{\omega_{ijt}\}_{ijt}$. And we denote the subset $\Omega=\{(i,j,t): \omega_{ijt}=1\}$. 
In this paper, we consider i.i.d. Bernoulli sampling with $\omega_{ijt}\overset{\text{i.i.d.}}{\sim} \text{Ber}(p)$. 
We propose the following estimator:
\begin{align}\label{proposed-estimator}
	\hat\calT := \arg\min_{\calT\in\RR^{M\times N\times T}} \frac{1}{2}\fro{\calT - p^{-1}\calP_{\Omega}(\calY)}^2 + \sum_{l=1}^K\lambda_{1,l}\nuc{\calT\times_3\f_l^\top} + \lambda_2\lone{\calT\times_3\F_1^\top}.
\end{align}
Here the nuclear norm and $\ell_1$ norm are the regularizations for the low-rank and sparse structures, respectively. 

We incorporate the knowledge of sampling distribution in the construction of the estimator by considering the expectation of $p^{-1}\calP_{\Omega}(\calT)$. This loss has been used in the low-rank matrix completion literature \cite{koltchinskii2011nuclear,mao2019matrix} as well as in low-rank tensor completion \cite{xia2021statistically}. 
	An alternative loss function widely used in the literature of low-rank tensor completion replaces the first term in \eqref{proposed-estimator} with $\frac{1}{2}\fro{p^{-1}\calP_{\Omega}(\calT) - p^{-1}\calP_{\Omega}(\calY)}^2$ \cite{xia2019polynomial,cai2022nonconvex,cai2022provable,song2023riemannian}. 
	
	The two approaches each have their advantages and limitations. Incorporating the sampling distribution into the estimator can yield more efficient estimators with closed-form expressions. In contrast, using the empirical loss typically leads to statistically optimal estimators but requires iterative algorithms and incurs higher computational cost. Since our primary focus is on computational efficiency for large-scale tensors, we adopt the first approach at the expense of a slight loss in statistical accuracy. We will discuss this trade-off in more detail shortly.

\subsection{Explicit Solution and Efficient Computations}
Now we discuss how $\hat\calT$ can be efficiently computed from \eqref{proposed-estimator}. While seemingly, the data-fitting term and the regularizations and in the original domain and Fourier domain, respectively, we can consider the problem in Fourier domain using the identity: $$\fro{\calT - p^{-1}\calP_{\Omega}(\calY)}^2 = \fro{\big(\calT - p^{-1}\calP_{\Omega}(\calY)\big)\times_3\F^\top}^2 = \sum_{t=1}^T\fro{\big(\calT - p^{-1}\calP_{\Omega}(\calY)\big)\times_3\f_l^\top}^2.$$
Then the original problem in \eqref{proposed-estimator} can be equivalently written as $K+1$ independent problems in the Fourier domain: 
\begin{subequations} 
	\begin{align}
			\hat\T_l &= \arg\min_{\T} \frac{1}{2}\fro{\T - p^{-1}\calP_{\Omega}(\calY)\times_3\f_l^\top}^2 + \lambda_{1,l}\nuc{\T_l}, \quad l = 1,\cdots, K\label{subprob:lowrank}\\
		\hat\bcalT_1 &= \arg\min_{\bcalT}\frac{1}{2}\fro{\bcalT - p^{-1}\calP_{\Omega}(\calY)\times_3\F_1^\top}^2 + \lambda_2\lone{\bcalT}. \label{subprob:sparse}
	\end{align}
\end{subequations}
Once we obtain $\hat\T_l, \hat\bcalT_1$, we first concatenate these matrices/tensor, then apply the inverse Fourier transform to recover the estimator for $\calT_0$:
\begin{subequations} 
\begin{align}
	\hat\bcalT &= \sum_{l=1}^K \hat\T_l\times_3\e_l + \hat\bcalT_1\times_3\mat{\e_{K+1}&\cdots&\e_{\lceil\frac{T+1}{2}\rceil}}\label{concatenateT}\\
	\hat\bcalT(:,:,l) &= \conj\big(\hat\bcalT(:,:,T+2-l)\big), \quad l = \lceil\frac{T+1}{2}\rceil+1, \cdots, T, \label{symmetrization}\\
	\hat\calT &= \hat\bcalT\times_3\conj(\F). \notag
\end{align}
\end{subequations}
For \eqref{subprob:lowrank}, we can use singular value shrinkage \cite{cai2010singular}, which admits a closed-form solution. Let the singular value soft-thresholding (SVT) operator:
\begin{align*}
	\calD_{\tau}(\M) = \U\diag(\{(\sigma_i-\tau)_+\})\V^*,
\end{align*}
where $\M$ admits the singular value decomposition $\M = \U\bSigma\V^*$ with $\bSigma = \diag(\{\sigma_i\})$, and $x_+ = \max\{x,0\}$. Then from \cite[Theorem 2.1]{cai2010singular}, we have 
\begin{align}\label{computeTl}
	\hat\T_l = \calD_{\lambda_{1,l}}\big(p^{-1}\calP_{\Omega}(\calY)\times_3\f_l^\top\big). 
\end{align}
On the other hand, for \eqref{subprob:sparse} we can compute $\hat\bcalT$ in closed-form entry-wisely.
Define $\calF_{\tau}(x) = (x_1-\tau)_++ i(x_2-\tau)_+$ for $x = x_1+ix_2\in\CC$. 
Then 
\begin{align}\label{computesparse}
	[\hat\bcalT_1]_{ijt} = \calF_{\lambda_2}\big([p^{-1}\calP_{\Omega}(\calY)\times_3\F_1^\top]_{ijt}\big).
\end{align}
We summarize the procedures for solving \eqref{proposed-estimator} in Algorithm \ref{alg:main}. Notice this algorithm can be performed in parallel. 
The computational cost of this algorithm is $O\big(MN(M\vee N)K + MN(T-K)\big)$. 
We also point out that our algorithm is also applicable to the low tubal-rank tensors, but with a computational cost of order $O\big(MN(M\vee N)T\big)$.
Since the number of low-rank slices $K$ is typically much smaller than the time dimension $T$, the proposed \name~ structure offers significant computational efficiency compared to the tubal-rank model.
\begin{algorithm}
	\caption{Parallel Estimation of $\mathcal{T}_0$}
	\begin{algorithmic}[1]
		\State \textbf{Input:} Observed data $\calP_{\Omega}(\calY)$, $p$, $\{\lambda_{1,l}\}_{l=1}^K$, $\lambda_2$
		\State \textbf{Output:} Estimator $\hat\calT_0$
		 
		\For{$l = 1$ to $K+1$ \textbf{in parallel}}\Comment{This step can be parallelized}
		\If{$l \leq K$}
		\State Compute $\hat\T_l$ via \eqref{computeTl}: $\hat\T_l = \calD_{\lambda_{1,l}}\big(p^{-1}\calP_{\Omega}(\calY)\times_3\f_l^\top\big)$
		\Else
		\State Compute $\hat\bcalT_1$ via \eqref{computesparse}: $\hat\bcalT_1 = \calF_{\lambda_2}\big(p^{-1}\calP_{\Omega}(\calY)\times_3\F_1^\top\big)$
		\EndIf
		\EndFor
		
		\State Aggregate $\hat\T_l$ and $\hat\bcalT_1$, then apply symmetrization to form $\hat\bcalT$ as in \eqref{concatenateT}, \eqref{symmetrization}
		\State Apply inverse Fourier transform to obtain $\hat\calT_0$: $\hat\calT_0 = \hat\bcalT\times_3\conj(\F)$
	\end{algorithmic}
\label{alg:main}
\end{algorithm}

\section{Theoretical Guarantee}
In this section, we present the theoretical guarantees for the proposed estimator in \eqref{proposed-estimator}. Our analysis establishes a sharp upper bound on the estimation error for \name~ tensor completion, which also improves existing results for low tubal-rank tensor estimation. The following Theorem \ref{thm:main} is proved in Section \ref{pf:thm:main}.
\begin{theorem}\label{thm:main}
	Assume the noise $\calE_{ijt}$ are independently distributed centered sub-Gaussian random variables with proxy variance $\sigma^2$, and $\max_{ijt}|[\calT_0]_{ijt}|\leq \gamma$. 
	If we choose 
	\begin{align*}
		\lambda_{1,l}&\geq C_1(\sigma\vee\gamma)\bigg(\frac{\sqrt{(N\vee M)\log(M\vee N)}}{\sqrt{p}} + \frac{\sqrt{\log^3(M\vee N)}}{p\sqrt{T}}\bigg),\\
		\lambda_2 &\geq C_2(\sigma\vee\gamma)\bigg(\frac{\sqrt{\log(M\vee N\vee T)}}{\sqrt{p}} + \frac{\log(M\vee N\vee T)}{\sqrt{T}p}\bigg).
	\end{align*}
	Then we have 
	\begin{align*}
	\fro{\hat\calT - \calT_0}^2 \leq 16(\sum_{l=1}^K\lambda_{1,l}^2r_l + \lambda_2^2s)
	\end{align*}
	with probability exceeding $1-2(M\vee N\vee T)^{-10}$.
\end{theorem}
The assumption on the noise and the upper bound on the tensor are standard and are widely considered in the low-rank matrix/tensor completion \cite{koltchinskii2011nuclear,xia2021statistically,cai2022provable}. 
We notice that the upper bound of the Frobenius norm estimation error $\fro{\hat\calT-\calT_0}^2$ depends on the choice of $\lambda_{1,l},\lambda_2$, and the upper bound is minimized when $\lambda_{1,l},\lambda_2$ match the lower bound in the theorem. We now give a more detailed interpretation of the result.
We write $p = \frac{n}{MNT}$, where $n = p\cdot MNT$ is the expected number of observation. Then with the smallest possible choice of $\lambda_{1,l}, \lambda_2$, we have 
\begin{align*}
	\frac{1}{MNT}\fro{\hat\calT - \calT_0}^2 \leq \tilde C\frac{\sum_{l=1}^K(M\vee N)r_l + s}{n}(\sigma\vee\gamma)^2 + \tilde C\frac{MN(\sum_{l=1}^Kr_l+s)}{n^2}(\sigma\vee\gamma)^2, 
\end{align*}
where $\tilde C$ depends only on $\log(M\vee N \vee T)$. 
\begin{corollary}\label{coro}
	Under the setting of Theorem \ref{thm:main}, assume $$n\geq \max\bigg\{(M\wedge N), \min\big\{(M\wedge N)\frac{s}{\sum_l r_l}, MN\big\}\bigg\},$$ we conclude 
	\begin{align*}
			\frac{1}{MNT}\fro{\hat\calT - \calT_0}^2 \leq \tilde C\frac{\sum_{l=1}^K(M\vee N)r_l + s}{n}(\sigma\vee\gamma)^2
	\end{align*}
	holds with probability exceeding $1-2(M\vee N\vee T)^{-10}$.
\end{corollary}
A special case arises when $K = \lceil \frac{T+1}{2} \rceil$, in which no sparsity is imposed in the model. In this setting, the proposed structure reduces to the standard low tubal-rank tensor completion problem. Our result then yields the following error bound:
\begin{align*}
    \frac{1}{MNT}\fro{\hat\calT - \calT_0}^2 \leq \tilde C\frac{(M\vee N)Tr_{\max}}{n}(\sigma\vee\gamma)^2, 
\end{align*}
where $r_{\max} = \max_{l=1}^{\lceil \frac{T+1}{2}\rceil} r_l$. This improves upon the existing bound in \cite{wang2019noisy}, which includes an additional $\sqrt{1/n} \gamma^2$ term. Moreover, our bound matches the minimax lower bound established in \cite[Theorem 2]{wang2019noisy}.

In the general case, the degrees of freedom of the proposed \name~ model are bounded by $\dof\leq \sum_{l=1}^K(M\vee N)r_l + s$. 
In practice, it suffices to choose $K\ll T$, and $s = O(1)$. 
Therefore, the sample complexity required in Corollary \ref{coro} is valid as long as $n\geq \dof$. 
Then the average error bound given by Corollary \ref{coro} can be written as 
\begin{align*}
    \frac{1}{MNT}\fro{\hat\calT - \calT_0}^2 \leq \tilde C\frac{(M\vee N)\sum_{l=1}^Kr_{l}+s}{n}(\sigma\vee\gamma)^2.
\end{align*}
Notably, this bound is independent of the time dimension $T$, and improves upon the tubal-rank error bound by a factor of $K/T$.

\section{Numerical Experiment} \label{sec:Exp}
In this section, we present numerical results that corroborate our theoretical findings. The experiments demonstrate that the proposed \name~estimator is both efficient and accurate, consistently outperforming the low tubal-rank model, particularly in capturing essential structures with fewer parameters. 
All experiments were conducted on a CPU-only node of an HPC cluster, with each run completing in under 3 CPU-hours and using less than 2.5 GiB of memory.
Throughout, we will use root‑mean‑square error (RMSE) on different sets $\Delta\subset[M]\times[N]\times[T]$ as measurements:
$$\text{RMSE}({\Delta}) =\frac{ \fro{\calP_{\Delta}(\hat\calT- \calT_0)}}{{\sqrt{|\Delta|}}},$$
where $\hat\calT$ refers to the estimators by different algorithms.

\subsection{Simulation Study}
We begin by presenting simulation results to evaluate our method on general tensor completion tasks and compare it with an existing low tubal-rank approach.
For each setting, we generate one ground‑truth tensor $\mathcal{T}_0\in\mathbb{R}^{M\times N\times T}$ that follows a specified $(r, K, s)$-\name~ structure ($(r,K,s)$-\textsf{FLoST} stands for the case when $r_1=\cdots=r_K=r$). 
The data generation proceeds as follows: we first draw a random tensor of size $M \times N\times T$ whose entries are i.i.d. standard normal, which is in general full rank. Then we take the Fourier transform along the third mode. For the first $K$ frequency slices, we truncate each slice to matrix rank $r$ using SVD. For slices $K+1$ through $\lceil (T+1)/2\rceil$, we keep only the $s$ largest entries and set the rest to zero to induce sparsity. We set the remaining slices to be the complex conjugates of their earlier counterparts so that the inverse transformed tensor is real‑valued on the original domain. Finally, we apply the inverse Fourier transform to obtain $\mathcal{T}_0$ with the desired $(r,K,s)$-\textsf{FLoST} structure. 

In the simulation, we add independent, pixel-wise Gaussian noise to every tensor entry, $\calE_{ijt}\sim\mathcal{N}(0,\sigma^{2})$, with $\sigma=0.1$. Observations are then drawn uniformly at random: each pixel is marked ``observed'' with probability $p=0.5$ (Panels A and B in Table \ref{tab:simu}) or $p=0.2$ (Panel C). Throughout the study, we fix the first two dimensions at $M=N=100$, the low-frequency rank at $r=5$, and consider three tensor depths $T\in\{100, 500, 1000\}$ with two sparsity cut-off proportions $K=T/10$ (Panels A and C in Table \ref{tab:simu}) and $K=T/20$ (Panel B); sparsity parameter $s\!=\!0.1 \times(\lceil\frac{T+1}{2}\rceil\!-\!K)MN$. Each configuration was replicated 100 times.

Under the simulation setting, we applied Algorithm \ref{alg:main} for \textsf{FLoST} in two variants: true-$K$ model ($\textsf{FLoST}$-1 in Table \ref{tab:simu}) with $K$ set to its data-generating value and no-sparsity model ($\textsf{FLoST}$-2) with $K=\lceil (T+1)/2\rceil$. 
For benchmarking, we use the Riemannian Conjugate Gradient Descent method (RCGD) for low tubal-rank tensor completion \cite{song2023riemannian}. 
RCGD is a non-convex algorithm that has been shown to be more computationally efficient than convex relaxation methods \cite{xu2019fast}, while still enjoying strong theoretical guarantees.
In particular, exact recovery is provably achievable under noiseless observations. 
It achieves high accuracy by exploiting the geometry of the low-tubal-rank Riemannian manifold and maintains computational efficiency through the use of conjugate gradient updates.
For this reason, we consider RCGD a strong baseline for evaluating tensor completion performance.
We implement Algorithm 2\footnote{To prevent slow convergence, we reserve 10 \% of the observed entries as a validation set and stop the iterations when its loss ceases to improve.} of \cite{song2023riemannian}, referring to it as RCGD in Table \ref{tab:simu}. The method is run with the true tubal rank, $r = 5$, matching the low-frequency rank used in data generation. 

Performance is evaluated using RMSE on both the training (observed) set $\Omega$ and the test (unobserved) set $\Omega^C$. In most cases, \textsf{FLoST}-1 (with true $K$) achieves the lowest test RMSE, which is expected since the data-generating process follows the $(r, K, s)$-\textsf{FLoST} structure. Both \textsf{FLoST}-1 and \textsf{FLoST}-2 consistently outperform the RCGD baseline in terms of test RMSE and runtime\footnote{Here, we implement Algorithm~\ref{alg:main} without parallelization.}. The runtime advantage of \textsf{FLoST} becomes more substantial as the temporal dimension $T$ increases. Notably, the similar test RMSEs of \textsf{FLoST}-1 and \textsf{FLoST}-2 suggest that the method is robust to over-parametrization in $K$, offering flexibility in practice. On the other hand, RCGD tends to achieve lower training RMSE but higher test RMSE, indicating overfitting. Overall, \textsf{FLoST} delivers more accurate and computationally efficient reconstructions, particularly for tensors with long temporal spans.


\begin{table}[H] \label{tab:simu}
\centering
\caption{Mean (± s.d.) test RMSE, train RMSE, and runtime (in seconds) for three tensor lengths $T$ across three parameter settings. \textsf{FLoST}-1 uses the true $K$; \textsf{FLoST}-2 sets $K$ to be maximum (no sparsity); RCGD is the existing baseline. Boldface highlights the best test RMSE or fastest runtime in each column (100 replicates).}

\small
\setlength{\tabcolsep}{5pt}
\textbf{Panel A:} $(p\!=\!0.5\;,K\!=\!T/10)$ \\[2pt]
\begin{tabular}{
  l                       
 *{9}{c} 
}
\toprule
Method &
\multicolumn{3}{c}{$T=100$} &
\multicolumn{3}{c}{$T=500$} &
\multicolumn{3}{c}{$T=1000$}\\
\cmidrule(lr){2-4}\cmidrule(lr){5-7}\cmidrule(lr){8-10}
 & {Test} & {Train} & {Time} &
   {Test} & {Train} & {Time} &
   {Test} & {Train} & {Time}\\
\midrule
\textsf{FLoST‑1} 
&\textbf{\meanstd{0.530052}{0.013304}}&\meanstd{0.515617}{0.025187}&\textbf{\runtime{0.132}}
&\textbf{\meanstd{0.525928}{0.014556}}&\meanstd{0.510799}{0.025437}&\textbf{\runtime{1.017}}
&\textbf{\meanstd{0.533788}{0.009990}}&\meanstd{0.524889}{0.017314}&\textbf{\runtime{2.185}}
 \\ \textsf{FLoST‑2} 
&\meanstd{0.539930}{0.000407}&\meanstd{0.529661}{0.000424}&\runtime{0.211}
&\meanstd{0.538740}{0.000179}&\meanstd{0.528569}{0.000249}&\runtime{1.222}
&\meanstd{0.538240}{0.000133}&\meanstd{0.528034}{0.000164}&\runtime{2.585}
 \\ RCGD 
&\meanstd{0.581062}{0.001249}&\meanstd{0.445475}{0.000400}&\runtime{1.090}
&\meanstd{0.578015}{0.001142}&\meanstd{0.443538}{0.000173}&\runtime{7.743}
&\meanstd{0.577011}{0.001105}&\meanstd{0.442880}{0.000130}&\runtime{16.725} \\
\bottomrule
\end{tabular}

\vspace{0.8em}
\textbf{Panel B:} $(p\!=\!0.5\;,K\!=\!T/20)$ \\[2pt]
\begin{tabular}{ l*{9}{c} }
\toprule
Method & \multicolumn{3}{c}{$T=100$} & \multicolumn{3}{c}{$T=500$} &\multicolumn{3}{c}{$T=1000$}\\
\cmidrule(lr){2-4}\cmidrule(lr){5-7}\cmidrule(lr){8-10}
 & {Test} & {Train} & {Time} & {Test} & {Train} & {Time} & {Test} & {Train} & {Time}\\
\midrule
\textsf{FLoST‑1} 
&\meanstd{0.558890}{0.000396}&\meanstd{0.556274}{0.000386}&\textbf{\runtime{0.108}}
&\meanstd{0.557608}{0.000178}&\meanstd{0.554913}{0.000190}&\textbf{\runtime{0.907}}
&\meanstd{0.556964}{0.000117}&\meanstd{0.554220}{0.000109}&\textbf{\runtime{2.017}}
 \\ \textsf{FLoST‑2} 
&\textbf{\meanstd{0.558735}{0.000393}}&\meanstd{0.550435}{0.000408}&\runtime{0.215}
&\textbf{\meanstd{0.557460}{0.000177}}&\meanstd{0.549246}{0.000222}&\runtime{1.256}
&\textbf{\meanstd{0.556811}{0.000116}}&\meanstd{0.548560}{0.000136}&\runtime{2.680}
\\ RCGD 
&\meanstd{0.618086}{0.001288}&\meanstd{0.469623}{0.000436}&\runtime{1.139}
&\meanstd{0.615001}{0.001282}&\meanstd{0.467674}{0.000173}&\runtime{8.036}
&\meanstd{0.613894}{0.001256}&\meanstd{0.466947}{0.000141}&\runtime{16.685}
\\ \bottomrule
\end{tabular}

\vspace{0.8em}
\textbf{Panel C:} $(p\!=\!0.2\;,K\!=\!T/10)$ \\[2pt]
\begin{tabular}{ l*{9}{c} }
\toprule
Method & \multicolumn{3}{c}{$T=100$} & \multicolumn{3}{c}{$T=500$} &\multicolumn{3}{c}{$T=1000$}\\
\cmidrule(lr){2-4}\cmidrule(lr){5-7}\cmidrule(lr){8-10}
 & {Test} & {Train} & {Time} & {Test} & {Train} & {Time} & {Test} & {Train} & {Time}\\
\midrule
\textsf{FLoST‑1}  
&\textbf{\meanstd{0.545575}{0.000190}}&\meanstd{0.539862}{0.000808}&\textbf{\runtime{0.146}}
&\textbf{\meanstd{0.544367}{0.000103}}&\meanstd{0.538453}{0.000390}&\textbf{\runtime{1.066}}
&\textbf{\meanstd{0.543857}{0.000079}}&\meanstd{0.537957}{0.000268}&\textbf{\runtime{2.198}}
 \\ \textsf{FLoST‑2} 
&\meanstd{0.545668}{0.000194}&\meanstd{0.527733}{0.000876}&{\runtime{0.229}}
&\meanstd{0.544471}{0.000101}&\meanstd{0.526758}{0.000430}&\runtime{1.234}
&\meanstd{0.543959}{0.000078}&\meanstd{0.526302}{0.000307}&\runtime{2.577}
\\ RCGD 
&\meanstd{0.633714}{0.001966}&\meanstd{0.377490}{0.000855}&\runtime{1.204}
&\meanstd{0.630746}{0.001839}&\meanstd{0.375577}{0.000420}&\runtime{8.075}
&\meanstd{0.629794}{0.001768}&\meanstd{0.375057}{0.000378}&\runtime{16.282} \\ 
\bottomrule
\end{tabular}
\end{table}

\subsection{Results for TEC Completion} \label{sec:tec}
We apply our method to the TEC completion problem introduced in Section \ref{sec:Introduction}. The data come from the open‑access VISTA TEC database \cite{sun2023complete,sunData}, which is fully imputed using the VISTA algorithm\cite{sun2022matrix}. Its spatial–temporal resolution is $1^{\circ}\text{(latitude)}\times 1^{\circ}\text{(longitude)}\times 5 \text{(minutes)}$.

In the experiment, we study a six‑day period, 21–26 June 2019. The TEC video in this time window forms a third‑order tensor of dimension $181\times361\times1728$. For each map, the second‑dimension coordinates are converted from longitudes to local time (LT), and the columns are circularly shifted so that noon (12 LT) always sits at the center. This sun‑facing view is more stationary because the highest TEC values stay clustered near the center of the maps. This stationarity fits the \textsf{FLoST} model well and helps leverage its efficiency. The resulting VISTA TEC tensor, with the time dimension $T=1728$, serves as the ground truth in our study. 

To test our method, we randomly mask 50\% of the ground‑truth tensor. Each pixel is removed independently with probability 0.5. This missing rate matches that observed in real TEC maps \cite{sun2022matrix}. Within the \textsf{FLoST} framework, we assume the ranks of the first $K$ low-frequency components are uniformly 5, and we tested three values of the sparsity cut‑offs $K\in\{100,400,865\}$. Because $\lceil\frac{T+1}{2}\rceil=865$, the choice $K=865$ corresponds to the degenerate model without sparsity. The loss‑function weights, $\{\lambda_{1,l}\}_{l=1}^K$ and $\lambda_2$, are tuned by $\texttt{Bayesian Optimization}$ (package \cite{BayesOpt}) on a validation set containing 10\% of the observed entries. 

Figure \ref{fig:TEC1} presents, for four representative timestamps, the imputed maps produced by \textsf{FLoST}, the ground‑truth VISTA maps, and the input maps with missing entries. These comparisons show that \textsf{FLoST} successfully recovers the main spatial patterns and that the performance is robust with different choices of the sparsity cut‑offs.

\begin{figure}[H]
    \centering
    \includegraphics[width=1\linewidth]{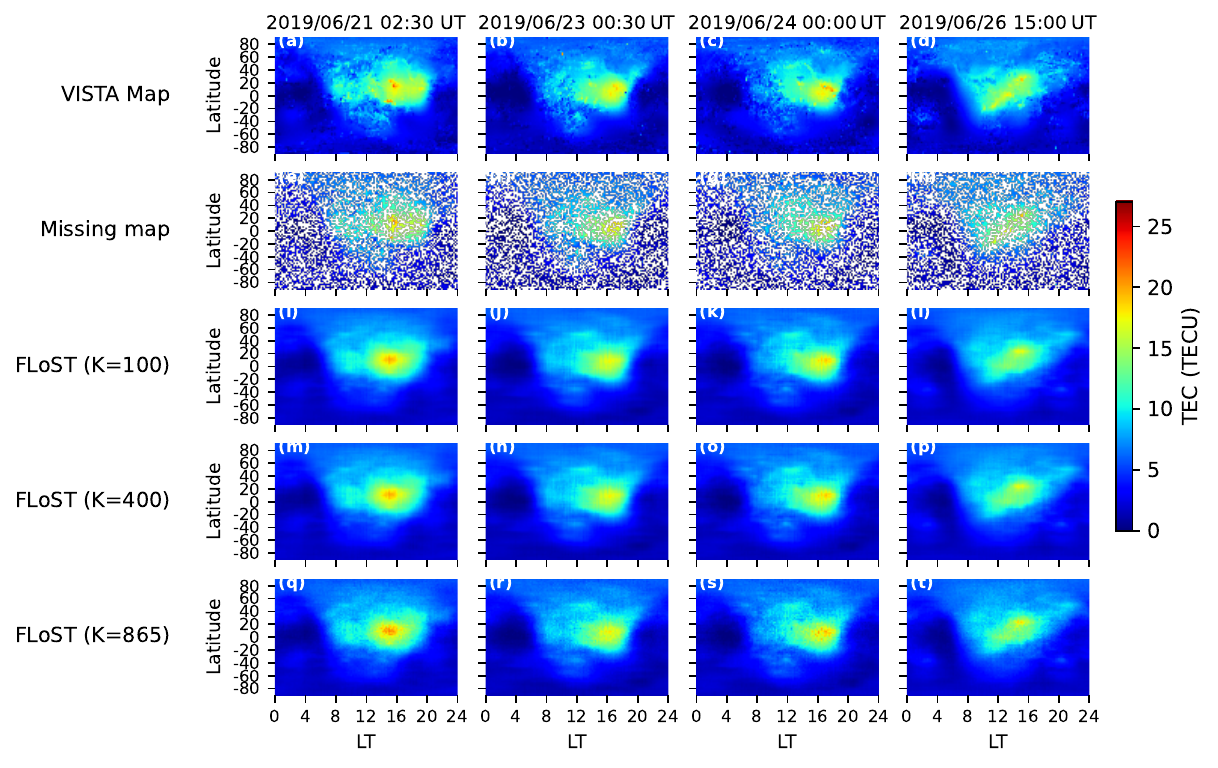}
    \caption{Total electron content (TEC) maps for four representative times between 21–26 June 2019. The VISTA reference (top) is compared with reconstructions from $\textsf{FLoST}$ with $K=100$, $K=400$ and $K=865$; where $K=865$ corresponds to the model without high‑frequency sparsity. Columns list UTC timestamps; horizontal axis is local time (LT), vertical axis latitude. Color bar gives TEC in TEC Unit (TECU).}
    \label{fig:TEC1}
\end{figure}

\begin{figure}[ht]
    \centering
    \includegraphics[width=1\linewidth]{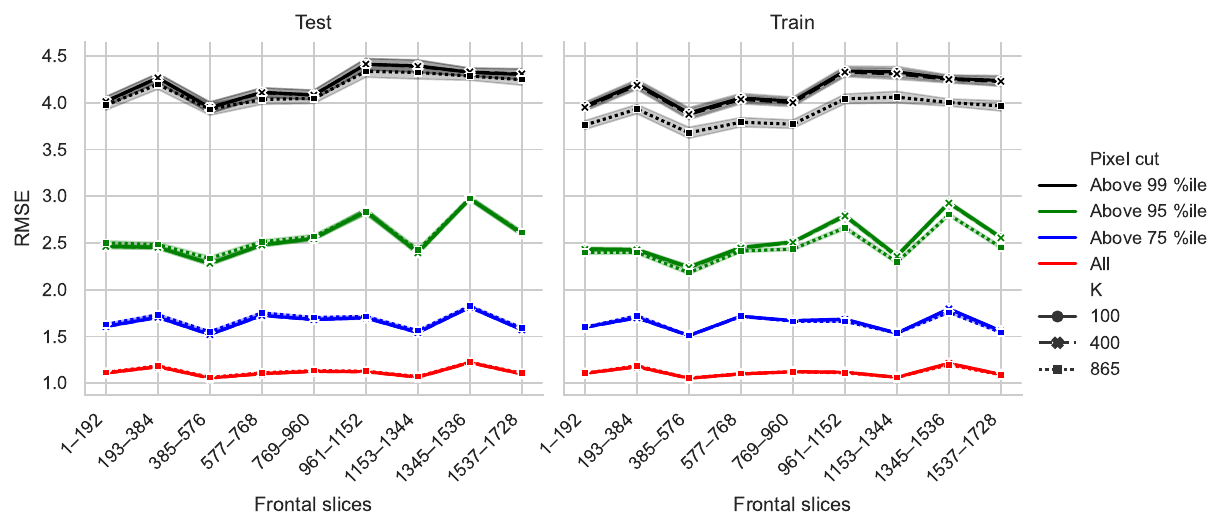}
    \caption{Root‑mean‑square error (RMSE) of the reconstructed tensor per 192‑frame interval. The two facets separate results on the training (observed) and test (missing) data. Colors indicate the subset of pixels used for RMSE evaluation: all pixels or only those whose true value exceeds the 75th, 95th, or 99th percentiles of the ground‑truth tensor. Line style encodes the $K$ value of $\textsf{FLoST}$; $K=865$ corresponds to the model without high‑frequency sparsity. Curves are means of 100 runs with ± 3 s.d. bands. For scale reference, the largest true‑tensor pixel value rounds to 31. }
    \label{fig:TEC2}
\end{figure}

To further evaluate performance, we report the RMSE on various pixel groups. 
To calculus RMSEs, we first cut the long tensor into nine chunks, each containing 192 timestamps, and then split the tensor into training (observed) and test (withheld) pixels. 
Within each split we isolate high‑TEC pixels, those above the 75th, 95th, and 99th empirical percentiles. RMSE is calculated on the pixels that are above each threshold. 
Figure \ref{fig:TEC2} shows that RMSE values computed on all pixels (threshold at the 0 \% quantile) are always below 1.5, which is small relative to the true maximum of about 31.

The wide gaps between the colored curves indicate that RMSE increases sharply as the evaluation threshold rises, meaning that underestimating the largest-magnitude entries contributes more heavily to the total error. Across different truncation settings, the curves for $K = 100$ and $K = 400$ almost entirely overlap, demonstrating that \textsf{FLoST} is robust to the choice of the sparsity cut-off $K$. The only notable deviation occurs for the non-sparse model ($K = 865$), which achieves slightly lower training RMSE on the ``above 99\%ile'' subset but exhibits no improvement in test RMSE compared to the sparser models. This suggests potential overfitting when using excessively large $K$, and indicates that introducing sparsity not only reduces the number of parameters but also improves generalization to unseen data.

\section{Conclusion and Discussion} 
In this work, we propose a new tensor structure, Fourier Low-rank and Sparse Tensor (\name), along with an efficient estimator for completing tensors that follow this model. By modeling low-frequency components as low-rank and high-frequency components as sparse, \name~removes the redundancy inherent in existing low tubal-rank models and better aligns with the physical characteristics of real-world spatiotemporal data such as total electron content (TEC) maps. We developed an efficient estimator for \name~ and established sharp theoretical guarantees. Our analysis recovers and improves existing results for noisy low tubal-rank tensor completion. Numerical experiments further confirm that \name~ achieves higher accuracy and lower complexity compared to baseline methods.

A limitation of the current work lies in the assumption on the missingness pattern—we assume that each tensor entry is observed independently with some fixed probability. However, real-world data, such as TEC maps, often exhibit structured missingness, where entire spatial or temporal blocks are missing due to sensor outages or coverage gaps. While there has been progress on developing algorithms that can handle such patterns, the theoretical understanding remains limited. Extending the analysis of \name~to accommodate structured missingness poses significant challenges and is a compelling direction for future research.

\bibliographystyle{plain}
\bibliography{reference.bib}

\appendix

\section{Proofs}
\subsection{Proof of Theorem \ref{thm:main}}\label{pf:thm:main}
\begin{proof}
	We have the following oracle inequality: 
	\begin{align*}
		&\quad\frac{1}{2}\fro{\hat\calT - p^{-1}\calP_{\Omega}(\calY)}^2 + \sum_{l=1}^K\lambda_{1,l}\nuc{\hat\calT\times_3\f_l^\top} + \lambda_2\lone{\hat\calT\times_3\F_1^\top}\\
		&\leq \frac{1}{2}\fro{\calT_0 - p^{-1}\calP_{\Omega}(\calY)}^2 + \sum_{l=1}^K\lambda_{1,l}\nuc{\calT_0\times_3\f_l^\top} + \lambda_2\lone{\calT_0\times_3\F_1^\top}
	\end{align*}
	Rearrange the terms and we get:
	\begin{align}\label{main:eq}
		&\frac{1}{2}\fro{\hat\calT-\calT_0}^2 \leq \inp{\hat\calT - \calT_0}{p^{-1}\calP_{\Omega}\calE + p^{-1}\calP_{\Omega}\calT_0 - \calT_0}\notag\\
		&\quad + \sum_{l=1}^K\lambda_{1,l}\big(\nuc{\calT_0\times_3\f_l^\top} - \nuc{\hat\calT\times_3\f_l^\top}\big) +\lambda_2\big(\lone{\calT_0\times_3\F_1^\top} - \lone{\hat\calT\times_3\F_1^\top}\big).
	\end{align}
	Notice we have 
	$
	\inp{\calA}{\calB} = \inp{\calA\times_3\F}{\calB\times_3\F} = \sum_{l=1}^T\inp{\calA\times_3\f_l^\top}{\calB\times_3\f_l^\top}.  
	$
	Let the compact SVD of $\calT_0\times_3\f_l^\top$ be $\calT_0\times_3\f_l^\top = \U_l\S_l\V_l^*$ for $l=1,\cdots, K$. And $\calS = \supp(\calT_0\times_3 \F_1^\top)$, $s = |\calS|$.
	Next we consider for $l=1,\cdots,K$,$$\inp{(\hat\calT - \calT_0)\times_3\f_l^\top}{\underbrace{\big(p^{-1}\calP_{\Omega}\calE+ p^{-1}\calP_{\Omega}\calT_0 - \calT_0\big)\times_3\f_l^\top}_{:=\Y_l}}.$$
	We have 
	\begin{align}\label{eq1.1}
		&\quad|\inp{(\hat\calT - \calT_0)\times_3\f_l^\top}{\Y_l}|\notag\\
		&\leq |\inp{\calP_{\TT_l}\big((\hat\calT - \calT_0)\times_3\f_l^\top\big)}{\calP_{\TT_l}\Y_l}| + |\inp{\calP_{\TT_l}^{\perp}\big((\hat\calT - \calT_0)\times_3\f_l^\top\big)}{\calP_{\TT_l}^{\perp}\Y_l}|\notag\\
		&\leq \fro{\calP_{\TT_l}\big((\hat\calT - \calT_0)\times_3\f_l^\top\big)}\cdot \fro{\calP_{\TT_l}\Y_l} + \nuc{\calP_{\TT_l}^{\perp}\big((\hat\calT - \calT_0)\times_3\f_l^\top\big)}\cdot\op{\calP_{\TT_l}^{\perp}\Y_l}\notag\\
		&\leq \fro{\calP_{\TT_l}\big((\hat\calT - \calT_0)\times_3\f_l^\top\big)}\cdot \sqrt{2r_l}\op{\Y_l} + \nuc{\calP_{\TT_l}^{\perp}\big(\hat\calT\times_3\f_l^\top\big)}\cdot\op{\Y_l}
	\end{align}
	On the other hand, from Lemma \ref{lemma:lowrank:subgrad}, we have 
	\begin{align}\label{eq1.2}
		\nuc{\calT_0\times_3\f_l^\top} - \nuc{\hat\calT\times_3\f_l^\top}\leq \nuc{\U_l^*\cdot(\calT_0 - \hat\calT)\times_3\f_l^\top\cdot\V_l} - \nuc{\calP_{\TT_l}^{\perp}\big(\hat\calT\times_3\f_l^\top\big)}.
	\end{align}
	Here we use the fact that $\calP_{\TT_l}^{\perp}(\cdot) = \calP_{\U_l}^{\perp}(\cdot)\calP_{\V_l}^{\perp}$. 
	
	Now we consider  
	$
	\inp{(\hat\calT - \calT_0)\times_3\F_1^\top}{\underbrace{\big(p^{-1}\calP_{\Omega}\calE+ p^{-1}\calP_{\Omega}\calT_0 - \calT_0\big)\times_3\F_1^\top}_{:=\bcalY_1}}.
	$
	In fact, 
	\begin{align}\label{eq2.1}
		&\quad|\inp{(\hat\calT - \calT_0)\times_3\F_1^\top}{\bcalY_1}| \notag\\
		&\leq  |\inp{\calP_{\calS}\big((\hat\calT - \calT_0)\times_3\F_1^\top\big)}{\calP_{\calS}(\bcalY_1)}| + |\inp{\calP_{\calS}^{\perp}\big((\hat\calT - \calT_0)\times_3\F_l^\top\big)}{\calP_{\calS}^{\perp}\bcalY_1}|\notag\\
		&\leq \fro{\calP_{\calS}\big((\hat\calT - \calT_0)\times_3\F_1^\top\big)}\cdot\fro{\calP_{\calS}(\bcalY_1)} + \lone{\calP_{\calS}^{\perp}\big(\hat\calT\times_3\F_1^\top\big)}\cdot\linf{\calP_{\calS}^{\perp}(\bcalY_1)}\notag\\
		&\leq \big(\sqrt{s}\fro{\calP_{\calS}\big((\hat\calT - \calT_0)\times_3\F_1^\top\big)} + \lone{\calP_{\calS}^{\perp}\big(\hat\calT\times_3\F_1^\top\big)}\big)\cdot\linf{\bcalY_1}\notag\\
		&\leq \big(\sqrt{s}\fro{(\hat\calT - \calT_0)\times_3\F_1^\top} + \lone{\calP_{\calS}^{\perp}\big(\hat\calT\times_3\F_1^\top\big)}\big)\cdot\linf{\bcalY_1},
	\end{align}
	where in the second inequality we use the fact that $|\inp{\x}{\y}|\leq \lone{\x}\cdot\linf{\y}$ for complex vectors $\x,\y$. 
	Moreover, from Lemma \ref{lemma:sparse:subgrad}, we have 
	\begin{align}\label{eq2.2}
		\lone{\calT_0\times_3\F_1^\top} - \lone{\hat\calT\times_3\F_1^\top}\leq \lone{\calP_{\calS}\big((\calT_0-\hat\calT)\times_3\F_1^\top\big)} - \lone{\calP_{\calS}^{\perp}\big(\hat\calT\times_3\F_1^\top\big)}.
	\end{align}
	Now we plug \eqref{eq1.1}-\eqref{eq2.2} into \eqref{main:eq}, and we get 
	\begin{align*}
		&\quad\frac{1}{2}\fro{\hat\calT-\calT_0}^2 + \sum_{l=1}^K\lambda_{1,l}\nuc{\calP_{\TT_l}^{\perp}\big(\hat\calT\times_3\f_l^\top\big)} + \lambda_2\lone{\calP_{\calS}^{\perp}\big(\hat\calT\times_3\F_1^\top\big)}\\
		&\leq \sum_{l=1}^K\lambda_{1,l}\sqrt{r_l}\fro{(\calT_0 - \hat\calT)\times_3\f_l^\top} + \lambda_2\sqrt{2s}\fro{(\calT_0-\hat\calT)\times_3\F_1^\top}\\
		&\quad + \sum_{l=1}^K2\op{\Y_l}\big(\sqrt{2r_l}\fro{(\hat\calT - \calT_0)\times_3\f_l^\top} + \nuc{\calP_{\TT_l}^{\perp}\big(\hat\calT\times_3\f_l^\top\big)}\big)\\
		&\quad + \linf{\bcalY_1}\big(\sqrt{s}\fro{(\hat\calT - \calT_0)\times_3\F_1^\top} + \lone{\calP_{\calS}^{\perp}\big(\hat\calT\times_3\F_1^\top\big)}\big). 
	\end{align*}
	Rearrange the terms, and we get, 
	\begin{align*}
		&\quad\frac{1}{2}\fro{\hat\calT-\calT_0}^2 + \sum_{l=1}^K\big(\lambda_{1,l}-2\op{\Y_l}\big)\nuc{\calP_{\TT_l}^{\perp}\big(\hat\calT\times_3\f_l^\top\big)} + \big(\lambda_2-\linf{\bcalY_1}\big)\lone{\calP_{\calS}^{\perp}\big(\hat\calT\times_3\F_1^\top\big)}\\
		&\leq \sum_{l=1}^K\big(\lambda_{1,l}+2\sqrt{2}\op{\Y_l}\big)\sqrt{r_l}\fro{(\calT_0 - \hat\calT)\times_3\f_l^\top} + \big(\sqrt{2}\lambda_2+ \linf{\bcalY_1}\big)\sqrt{s}\fro{(\calT_0-\hat\calT)\times_3\F_1^\top}.
	\end{align*}
	
	Finally, as long as we choose $\frac{1}{2}\lambda_{1,l}\geq 2\op{\Y_l}$, $\frac{1}{2}\lambda_2\geq \linf{\bcalY_1}$, we get 
	\begin{align*}
		\frac{1}{2}\fro{\hat\calT-\calT_0}^2 
		&\leq \sum_{l=1}^K2\lambda_{1,l}\sqrt{r_l}\fro{(\calT_0 - \hat\calT)\times_3\f_l^\top} + 2\lambda_2\sqrt{s}\fro{(\calT_0-\hat\calT)\times_3\F_1^\top}\\
		&\leq (\sum_{l=1}^K4\lambda_{1,l}^2r_l + 4\lambda_2^2s)^{1/2}(\fro{(\calT_0 - \hat\calT)\times_3\f_l^\top}^2+ \fro{(\calT_0-\hat\calT)\times_3\F_1^\top}^2)^{1/2}\\
		&\leq (\sum_{l=1}^K4\lambda_{1,l}^2r_l + 4\lambda_2^2s)^{1/2}\fro{\hat\calT-\calT_0},
	\end{align*}
	where the second inequality is due to Cauchy-Schwarz inequality. 
	Notice the condition $\frac{1}{2}\lambda_{1,l}\geq 2\op{\Y_l}$, $\frac{1}{2}\lambda_2\geq \linf{\bcalY_1}$ can be satisfied with high probability from Lemma \ref{lemma:Y}.
	And this finishes the proof. 

\end{proof}

\section{Technical Lemmas}
	\begin{lemma}\label{lemma:lowrank:subgrad}
	Let $\X_0\in\RR^{M\times N}$ be a rank $r$ matrix with the compact SVD $\X_0 = \U\bSigma \V^*\subset[M]\times[N]$, we have 
	\begin{align*}
		\nuc{\X_0} - \nuc{\X} \leq \nuc{\U^*(\X - \X_0)\V} - \nuc{\calP_{\U}^{\perp}(\X - \X_0)\calP_{\V}^{\perp}}
	\end{align*}
	for any $\X\in\RR^{M\times N}$.
\end{lemma}
\begin{proof}
	Let $\G\in\partial\nuc{\X_0}$, then $\G = \U\V^* + \W$, where $\op{\W}\leq 1$, $\U^*\W = 0$, $\W^*\V= 0$. We denote $\calP_{\U}(\M) = \U\U^*\M$ for orthogonal matrix $\U$, and $\calP_{\U}^{\perp}(\M) = (I-\U\U^*)\M$. 
	Then we have 
	\begin{align*}
		\nuc{\X} - \nuc{\X_0} &\geq \inp{\X-\X_0}{\G} = \inp{\calP_{\U}(\X - \X_0)\calP_{\V}}{\U\V^*} + \inp{\calP_{\U}^{\perp}(\X - \X_0)\calP_{\V}^{\perp}}{\W}\\
		&\geq -\nuc{\U^*(\X - \X_0)\V} + \nuc{\calP_{\U}^{\perp}(\X - \X_0)\calP_{\V}^{\perp}},
	\end{align*}
	where in the last inequality we use the fact $|\inp{\A}{\B}|\leq \op{\A}\cdot\nuc{\B}$, and $\calP_{\U}^{\perp}(\X - \X_0)\calP_{\V}^{\perp}(\X_0) = 0$ and we can choose $\G$ such that $\inp{\calP_{\U}^{\perp}\X\calP_{\V}^{\perp}}{\W}=\nuc{\calP_{\U}^{\perp}(\X - \X_0)\calP_{\V}^{\perp}}$. 
\end{proof}

\begin{lemma}\label{lemma:sparse:subgrad}
	Let $\X_0\in\RR^{M\times N}$ be with $\supp(\X_0) = \calS\subset[M]\times[N]$, we have 
	\begin{align*}
		\lone{\X_0} - \lone{\X} \leq \lone{\calP_{\calS}(\X_0-\X)} - \lone{\calP_{\calS}^{\perp}\X}
	\end{align*}
	for any $\X\in\RR^{M\times N}$.
\end{lemma}
\begin{proof}
	Let $\G\in\partial\lone{\X_0}$, then $\G = \sign(\X_0)$, where $\sign(0) \in [-1,1]$, and thus $\linf{\G} \leq  1$. 
	Then we have 
	\begin{align*}
		\lone{\X} - \lone{\X_0} &\geq \inp{\X-\X_0}{\G} = \inp{\calP_{\calS}(\X - \X_0)}{\G} + \inp{\calP_{\calS}^{\perp}(\X - \X_0)}{\calP_{\calS}^{\perp}(\G)}\\
		&\geq -\lone{\calP_{\calS}(\X - \X_0)} + \lone{\calP_{\calS}^{\perp}(\X)},
	\end{align*}
	where in the last inequality we use the fact $|\inp{\A}{\B}|\leq \linf{\A}\cdot\lone{\B}$, and $\calP_{\calS}^{\perp}(\X_0) = 0$ and we can choose $\G$ such that $\inp{\calP_{\calS}^{\perp}(\X)}{\calP_{\calS}^{\perp}(\G)} =  \lone{\calP_{\calS}^{\perp}(\X)}$. 
\end{proof}
\begin{lemma}\label{lemma:Y}
	We have for all $l=1,\cdots, K$, 
	\begin{align*}
	    \op{\Y_l}\lesssim \frac{\gamma\vee \sigma}{\sqrt{p}}\sqrt{N\vee M}\sqrt{\log(M\vee N\vee T)} + \frac{\gamma\vee\sigma}{p}\frac{1}{\sqrt{T}}\log^{3/2}(M\vee N\vee T)
	\end{align*}
	holds with probability exceeding $1- (M\vee N\vee T)^{-10}$. 
	And 
	\begin{align*}
		\linf{\bcalY_1} \lesssim (\sigma\vee\gamma)\bigg(\frac{\sqrt{\log(M\vee N\vee T)}}{\sqrt{p}} + \frac{\log(M\vee N\vee T)}{\sqrt{T}p}\bigg)
	\end{align*}
	with probability exceeding $1-(M\vee N\vee T)^{-10}$.
\end{lemma}
\begin{proof}
	For the first part, we can expand $\Y_l$ as follows: 
	\begin{align*}
		\Y_l = \underbrace{p^{-1}\sum_{i,j}\sum_tf_{lt}\omega_{ijt}\epsilon_{ijt}\cdot e_ie_j^\top}_{=:\Y_{l1}} + \underbrace{\sum_{i,j}\sum_tf_{lt}(p^{-1}\omega_{ijt}[\calT_0]_{ijt}-[\calT_0]_{ijt})\cdot e_ie_j^\top}_{=:\Y_{l2}} ,
	\end{align*}
	where $f_{lt} = \f_l^\top e_t$ has magnitude $|f_{lt}| = \frac{1}{\sqrt{T}}$. 
	Then we have 
	\begin{align*}
		\Y_{l1} = \sum_{ijt}p^{-1}f_{lt}\omega_{ijt}\epsilon_{ijt}\cdot e_ie_j^\top.
	\end{align*}
	And thus $\big\|\op{p^{-1}f_{lt}\omega_{ijt}\epsilon_{ijt}\cdot e_ie_j^\top}\big\|_{\psi_2}\leq \frac{\sigma}{p\sqrt{T}}$. For the variance, we have 
	\begin{align*}
		\frac{1}{MNT}\sum_{ijt}\EE p^{-2}|f_{lt}|^2\omega_{ijt}\epsilon_{ijt}^2\cdot e_ie_j^\top e_je_i^\top = \frac{1}{MNT}\frac{1}{p}\sigma^2N I_M.
	\end{align*}
	And similarly
	\begin{align*}
		\frac{1}{MNT}\sum_{ijt}\EE p^{-2}|f_{lt}|^2\omega_{ijt}\epsilon_{ijt}^2\cdot e_je_i^\top e_ie_j^\top  = \frac{1}{MNT}\frac{1}{p}\sigma^2M I_N.
	\end{align*}
	Therefore from matrix Bernstein's inequality (e.g., \cite[Proposition 2]{koltchinskii2011nuclear}), we conclude with probability exceeding $1-(M\vee N\vee T)^{-20}$, 
	\begin{align*}
		\op{\Y_{l1}}&\lesssim \frac{\sigma}{\sqrt{p}}\sqrt{N\vee M}\sqrt{\log(M\vee N\vee T)} + \frac{\sigma}{p}\frac{1}{\sqrt{T}}\log^{3/2}(M\vee N\vee T). 
	\end{align*}
	On the other hand, we have 
	$
	\Y_{l2} =\sum_{i,j}\sum_t f_{lt}[\calT_0]_{ijt}(p^{-1}\omega_{ijt}-1)\cdot e_ie_j^\top. 
	$
	And 
	\begin{align*}
		\op{f_{lt}[\calT_0]_{ijt}(p^{-1}\omega_{ijt}-1)\cdot e_ie_j^\top} \leq \frac{\gamma}{p\sqrt{T}}. 
	\end{align*}
	For the variance, we have 
	\begin{align*}
		\frac{1}{MNT}\sum_{ijt}\EE |f_{lt}|^2[\calT_0]_{ijt}^2(p^{-1}\omega_{ijt}-1)^2\cdot e_ie_j^\top e_je_i^\top \leq \frac{1}{MNT}\frac{1}{p}\gamma^2N I_M, 
	\end{align*}
	and 
	\begin{align*}
		\frac{1}{MNT}\sum_{ijt}\EE |f_{lt}|^2[\calT_0]_{ijt}^2(p^{-1}\omega_{ijt}-1)^2\cdot  e_je_i^\top e_ie_j^\top \leq \frac{1}{MNT}\frac{1}{p}\gamma^2M I_N. 
	\end{align*}
	So we conclude from matrix Bernstein's inequality \cite[Proposition 1]{koltchinskii2011nuclear}, with probability exceeding $1-(M\vee N\vee T)^{-20}$, 
	\begin{align*}
		\op{\Y_{l2}}\lesssim \frac{\gamma}{\sqrt{p}}\sqrt{N\vee M}\sqrt{\log(M\vee N\vee T)} + \frac{\gamma}{p}\frac{1}{\sqrt{T}}\log(M\vee N\vee T).
	\end{align*}
	In conclusion, for $l=1,\cdots, K$, 
	\begin{align*}
		\op{\Y_l} \lesssim \frac{\gamma\vee \sigma}{\sqrt{p}}\sqrt{N\vee M}\sqrt{\log(M\vee N\vee T)} + \frac{\gamma\vee\sigma}{p}\frac{1}{\sqrt{T}}\log^{3/2}(M\vee N\vee T).
	\end{align*}
	
	For the second part, we define the events:
	\begin{align*}
		\calE_{ij} &= \big\{\frac{1}{T}\sum_t\omega_{ijt}\leq C(p + T^{-1}\log(M\vee N\vee T))\big\},
	\end{align*}
	then from Chernoff bound, $\PP(\calE_{ij})\geq 1-(M\vee N\vee T)^{-20}$. 
	
	We can decompose $\bcalY_1$ as follows:
	\begin{align*}
		\bcalY_1 = \underbrace{p^{-1}\calP_{\Omega}(\calE) \times_3 \F_1^\top}_{:=\bcalZ_1} + \underbrace{(p^{-1}\calP_{\Omega}(\calT_0) - \calT_0) \times_3 \F_1^\top}_{:=\bcalZ_2}.
	\end{align*}
	And
	\begin{align*}
		\linf{\bcalZ_{1}} = \max_{ijl}\big|p^{-1}\sum_tf_{lt}\omega_{ijt}\epsilon_{ijt}\big|.
	\end{align*}
	Condition on $\omega_{ijt}$, we have 
	\begin{align*}
		\psitwo{p^{-1}\sum_tf_{lt}\omega_{ijt}\epsilon_{ijt}}^2\leq p^{-2}\sigma^2T^{-1}\sum_t\omega_{ijt},
	\end{align*}
	which has the following upper bound under $\calE_{ij}$:
	\begin{align*}
		\psitwo{p^{-1}\sum_tf_{lt}\omega_{ijt}\epsilon_{ijt}}^2\leq Cp^{-2}\sigma^2\bigg(p + \frac{\log(M\vee N\vee T)}{T}\bigg), 
	\end{align*}
	which leads to, with probability exceeding $1-(M\vee N\vee T)^{-20}$, 
	\begin{align*}
		|p^{-1}\sum_tf_{lt}\omega_{ijt}\epsilon_{ijt}|\leq Cp^{-1}\sigma\big(\sqrt{p} + \sqrt{T^{-1}\log(M\vee N\vee T)}\big)\sqrt{\log(M\vee N\vee T)}. 
	\end{align*}
	Taking union bound and we conclude with probability exceeding $1-MNT(M\vee N\vee T)^{-20}$, 
	\begin{align*}
		\linf{\bcalZ_1}\leq C\sigma\bigg(\frac{\sqrt{\log(M\vee N\vee T)}}{\sqrt{p}} + \frac{\log(M\vee N\vee T)}{\sqrt{T}p}\bigg).
	\end{align*}
	For $\bcalZ_2$, similarly we have 
	\begin{align*}
		\linf{\bcalZ_2} = \max_{ijl}\big|\sum_t(p^{-1}\omega_{ijt}-1)[\calT_0]_{ijt}f_{lt}\big|. 
	\end{align*}
	From Bernstein inequality, we have with probability exceeding $1-(M\vee N\vee T)^{-20}$,
	\begin{align*}
		\big|\sum_t(p^{-1}\omega_{ijt}-1)[\calT_0]_{ijt}f_{lt}\big| \leq C\bigg(\frac{\gamma}{\sqrt{p}}\sqrt{\log(M\vee N\vee T)}  +\frac{\gamma}{p\sqrt{T}}\log(M\vee N\vee T)\bigg)
	\end{align*}
	Taking union bound, and we conclude with probability exceeding $1-MNT(M\vee N\vee T)^{-20}$, 
	\begin{align*}
		\linf{\bcalZ_2}\leq  C\gamma\bigg(\frac{\sqrt{\log(M\vee N\vee T)}}{\sqrt{p}} + \frac{\log(M\vee N\vee T)}{\sqrt{T}p}\bigg). 
	\end{align*}
\end{proof}

\end{document}